\documentclass[conference, pdftex]{IEEEtran}
\IEEEoverridecommandlockouts
\usepackage{xcolor}
\usepackage{graphicx}
\usepackage[font=small]{caption}
\usepackage{amsmath, bbm}
\usepackage{amssymb}
\usepackage{amsthm}
\usepackage{booktabs}
\usepackage{multirow}
\usepackage{array,multirow}
\usepackage[numbers]{natbib}
\usepackage{textcomp}
\usepackage{hyperref}
\usepackage{algorithm,algorithmic}
\usepackage{cancel}
\usepackage{tikz}
\usetikzlibrary{arrows, arrows.meta}
\usepackage{setspace}

\newtheorem{theorem}{Theorem}

\def\BibTeX{{\rm B\kern-.05em{\sc i\kern-.025em b}\kern-.08em
    T\kern-.1667em\lower.7ex\hbox{E}\kern-.125emX}}

\DeclareMathAlphabet\mathbfcal{OMS}{cmsy}{b}{n}

\begin{document}

\title{Rate-Distortion in Image Coding for Machines
\thanks{\textsuperscript{*}Correspondence to \href{mailto:aharell@sfu.ca}{aharell@sfu.ca}.  Supported in part by Intel Labs.}
}

\author{\IEEEauthorblockN{Alon Harell\IEEEauthorrefmark{1}, Anderson De Andrade, and Ivan V. Baji\'{c}}
\IEEEauthorblockA{School of Engineering Science, Simon Fraser University, Burnaby, BC, Canada }
}

% \author{\IEEEauthorblockN{Ivan V. Baji\'{c}}
% \IEEEauthorblockA{School of Engineering Science, Simon Fraser University, Burnaby, BC, Canada \\
% ibajic@ensc.sfu.ca}
% }

\maketitle

\begin{abstract}
In recent years, there has been a sharp increase in transmission of images to remote servers specifically for the purpose of computer vision. In many applications, such as surveillance, images are mostly transmitted for automated analysis, and rarely seen by humans. Using traditional compression for this scenario has been shown to be inefficient in terms of bit-rate, likely due to the focus on human based distortion metrics. Thus, it is important to create specific image coding methods for joint use by humans and machines. One way to create the machine side of such a codec is to perform feature matching of some intermediate layer in a Deep Neural Network performing the machine task. In this work, we explore the effects of the layer choice used in training a learnable codec for humans and machines. We prove, using the data processing inequality, that matching features from deeper layers is preferable in the sense of rate-distortion. Next, we confirm our findings empirically by re-training an existing model for scalable human-machine coding. In our experiments we show the trade-off between the human and machine sides of such a scalable model, and discuss the benefit of using deeper layers for training in that regard.
\end{abstract}


\section{Introduction}
In recent years, Deep Neural Networks (DNNs) have become the preferred method for solving many computer vision (CV) problems~\cite{CHAI2021100134}. %We can see that in \textcolor{red}{[ada] tasks such as image classification, object detection...} image classifications ~\cite{some_classification_stuff}, object detections~\cite{Redmon2018_YOLOv3,ren2015faster}, semantic segmentation~\cite{segmentation_stuff}, and many more ~\cite{depth,pose}. 
At the same time, %\textcolor{red}{[ada] the increased availability of ... has enabled the time and space complexity of the ... to grow tremendously.}  due in part to the increased availability of computing resources (commonly referred to as \textit{compute}), 
the complexity in both memory and floating point operations of the aforementioned models has grown tremendously~\cite{growth}.  That growth, alongside the increasing variety of CV models and their uses, poses a challenge for deployment to low-resource edge devices such as mobile phones, smart speakers, etc. For this reason, in many CV tasks, the most common approach today~\cite{kang2017neurosurgeon} is to simply avoid deploying resource-intensive models to edge devices. Instead, the bulk of the computation is performed on remote servers (the cloud) equipped with powerful graphics processing units (GPU) or tensor processing units (TPU).

% \textcolor{blue}{[Alon] maybe I can just get rid of this overview and just say that a popular approach for this perform inference on the cloud}
% Several different approaches can be employed to face the practical challenge of deploying a \textcolor{red}{[ada] resource-intensive} power CV model to an end user. One approach is to directly reduce the complexity of the \textcolor{red}{[ada] model} networks in the design process. This reduction can be the result of changes to \textcolor{red}{[ada] the model architecture, the number of parameters, or the use ...}  model architecture~\cite{something_like_tinynet} or the use of lower numerical precision~\cite{dos2019impact}. An alternative solution involves the simplification of existing models to match \textcolor{red}{[ada] the computational capabilities of edge devices}  edge computational capabilities~\cite{edgification}. Although these \textcolor{red}{[ada] and other}  methods, and many more, hold great promise, the most common approach today~\cite{something_to_back_this} is to simply avoid deploying \textcolor{red}{[ada] resource-intensive models to edge devices} heavy models the edge. Instead, the bulk of the computation is performed on remote servers (the cloud) equipped with powerful graphics processing units (GPU) or tensor processing units (TPU).

%transmit images to a remote-server, where the bulk of computation can be performed on compute heavy computers equipped with graphics processing units (GPU) or tensor processing units (TPU).

Transmitting an image for inference on the cloud using traditional coding methods, which have been designed to accommodate human perception, has been shown to be sub-optimal both theoretically~\cite{hyomin} and empirically~\cite{hyomin,kang2017neurosurgeon}. Collaborative intelligence (CI)~\cite{kang2017neurosurgeon} suggests dividing a given CV model to a frontend, deployed on the edge, and backend, which remains on the cloud. Inference can then be started on the edge, creating intermediate representations at some layer $l$ of the CV model, which we denote $\mathcal{F}^{(l)}$. %for some layer $l$. 
These are then sent to the cloud backend, where inference is finished. Of course, this approach requires defining an efficient method for the communication of said representations (generally floating point tensors) to the cloud, which has been explored in several subsequent works~\cite{saeed_multi_task_learning,Choi2018NearLosslessDF,choi2018deep}.

Inspired by CI, the work of \cite{hyomin} proposes replacing the CV frontend with an encoder from a learnable image coding backbone~\cite{cheng2020learned}. This replaces the floating point tensors with highly compressible latent features, which are named the \emph{base} representation. On the decoder side, these can be used to match the intermediate features for the CV backend, which can then perform inference. The matching of intermediate CV features from the compressible latent space is performed by a learned transform named the \emph{Latent Space Transform} (LST). An important aspect of~\cite{hyomin}, is that it allows for human vision of the analysed image in a scalable manner. The \emph{base} representation is combined with a second representation, named \emph{enhancement}, to reconstruct the original image on the cloud side. In terms of rate-distortion, this method outperforms traditional as well as learnable coding schemes for the CV task, without significant degradation with regards to reconstruction for humans.

In this work we further investigate the behaviour of ~\cite{hyomin} in the context of both object detection and input reconstruction. We provide theoretical proof that performing feature matching on deeper layers of a CV model is beneficial %for collaborative intelligence 
in the sense of rate-distortion. Based on our proofs, we propose changes to the training loss of~\cite{hyomin} to improve rate-distortion behaviour for object detection. Using the proposed method, we achieve superior rate-distortion performance for the object detection task, which we consider to be empirical evidence of the theory presented in our proofs. Finally, we explore the trade-off in rate-distortion performance between computer and human vision in scalable compression, and discuss the consequences of our approach in that regard.

\section{Background}
\label{sec:background}
\subsection{Feature Matching in Collaborative Intelligence}
\label{subsec:man-machine}

% The computation of DNN-based computer vision models can generally be viewed as follows. Using a set of operations (both linear and non-linear), the model produces latent representations (often known as features), which we will denote $\mathcal{F}$ of an input $\mathbf{X}$. Commonly, a small subset of the operations (such as a convolution followed by a non-linear activation) is known as a layer. In such cases we denote the latent features at the output of some layer $l$ as $\mathcal{F}^{(l)}$, where $l = 1,2,...,L$, with $L$ being the total number of layers in the model. 

When splitting a model for use in CI, it is important to distinguish between models (known as single-stream) where layer operations are performed purely sequentially, and ones (known as multi-stream) where layers are combined in complex architectures including skip-connections and multi-resolution computations~\cite{Redmon2018_yolov3, ren2015faster, he2017mask}.
Often, a CV model will be comprised of some single-stream layers followed by a multi-stream section. We consider the output of a certain layer, $\mathcal{F}^{(l)}$, to be a single stream feature if we can perform inference by obtaining $\mathcal{F}^{(l)}$ from the input, and then passing it on to the rest of the model. Formally, we require $T = G_l(\mathcal{F}^{(l)})$, where $T$ is the CV task and $G_l(\cdot)$ is the operation described by all subsequent layers after $l$. Conversely, if the output of the task model cannot be inferred solely from  $\mathcal{F}^{(l)}$, we consider it a multi-stream feature.

In collaborative intelligence, the intermediate features produced on the edge device must be transmitted (or recreated from a latent representation) to be used downstream by the CV backend. Formally, we need to match, on the server side, the features at some small subset $K \ll L$, of the model's $L$ layers %\textcolor{red}{[ada] $\mathcal{F}^{(k)}; k = 1,2,...,K$} 
$\mathcal{F}^{(l_k)}$ where $k = 1,2,..,K$. Following our definitions above, we see that in the case $K=1$, we require $\mathcal{F}^{(l_1)}$ to be a single-stream feature. In most multi-stream models, this is often only achievable by selecting a relatively early layer. Unfortunately, as we will prove below, matching earlier features comes at a cost in terms of rate-distortion performance.

\subsection{Compression and Rate-Distortion Analysis}

Compression, in general, can be divided into lossless compression, where a source random variable (RV) $\mathbf{X}$ is reconstructed perfectly after decompressing; and lossy compression, where after decoding we are left with an approximation $\widehat{\mathbf{X}}$ (approximations or quantizations are denoted with a hat operator). For a single pair of observations $(\mathbf{x},\hat{\mathbf{x}})$, we measure the amount of inaccuracy introduced by the approximation using some distortion metric $d(\mathbf{x},\hat{\mathbf{x}})$. This leads to an important concept in lossy compression known as the \emph{rate-distortion function} $R(D)$~\cite{Cover_Thomas_2006}, given by the following:
\begin{equation}
    R(D) = \min_{p(\hat{\mathbf{x}}|\mathbf{x})~:~\mathbb{E} \left[d(\mathbf{X},\widehat{\mathbf{X}})\right]~\leq~D} I(\mathbf{X};\widehat{\mathbf{X}}).
    \label{eq:RD_function} %_{\mathbf{X},\widehat{\mathbf{X}}}
\end{equation}
% In the above equation, $\widehat{\mathbf{X}}$ is the quantized version of $\mathbf{X}$ -- this process of quantization, where values of $\mathbf{X}$ are represented by another, generally smaller set of values $\widehat{\mathbf{X}}$, is what causes ``loss'' in lossy compression. For a given pair  $(\mathbf{x},\hat{\mathbf{x}})$, the discrepancy is measured by a distortion metric $d(\mathbf{x},\hat{\mathbf{x}})$. %; some popular choices for the distortion metric are Hamming distance and squared error distortion, but others are possible~\cite{Cover_Thomas_2006}.
%\textcolor{red}{[ada] It is weird for me that the distortion function is defined over realizations of X and $\hat{X}$ but in equation 1 it is not. [Alon] This because a function when we take the expectations we must use RVs, even when the function is on the realizations.}
Here, $p(\hat{\mathbf{x}}|\mathbf{x})$ is the conditional distribution of the approximation given the source; $\mathbb{E}[d(\mathbf{X},\widehat{\mathbf{X}})]$ is the expected distortion with respect to the joint distribution $p(\mathbf{x},\hat{\mathbf{x}})$); $I(\mathbf{X};\widehat{\mathbf{X}})$ is the mutual information between $\mathbf{X}$; and $\widehat{\mathbf{X}}$; and $D$ is some value of distortion. 

Because the marginal distribution of the source, $p(\mathbf{x})$ is fixed, the joint distribution, $p(\mathbf{x},\hat{\mathbf{x}})$ only changes through the conditional distribution of the approximation $p(\hat{\mathbf{x}}|\mathbf{x})$. 
Using that, we can understand the minimization in the rate distortion function as finding an approximation $\widehat{\mathbf{X}}$ that gives the lowest mutual information with the source $\mathbf{X}$ while allowing an expected distortion no greater than $D$. Then, using the source coding theorem~\cite{Cover_Thomas_2006} and its converse, it can be shown that the resulting bit-rate\footnote{This assumes we calculate mutual information using $\log_2(\cdot)$.} is the lowest achievable rate (per source symbol) giving an expected distortion bound by $D$. In this sense, $R(D)$ is a fundamental bound on the performance of lossy compression, similarly to entropy being a bound on lossless compression. 

%The quantizer can be deterministic, in which case $p(\hat{\mathbf{x}}|\mathbf{x})$ is a delta function for any given $\mathbf{x}$, or random, in which case $p(\hat{\mathbf{x}}|\mathbf{x})$ is a proper, non-degenerate distribution. The theory is general enough to accommodate both cases. 
% The source coding theorem and its converse~\cite{Cover_Thomas_2006} show that $R(D)$, as defined above, is the minimum achievable rate\footnote{Rate is in \emph{bits} if $\log_2$ is used in the computation of $I(\mathbf{X};\widehat{\mathbf{X}})$.} per source symbol that results in expected distortion of at most $D$. Hence, $R(D)$ represents a fundamental bound on lossy compression, just like entropy represents a fundamental bound on lossless compression.

\section{Theoretical Discussion}
% \color{blue}
As explained in Section~\ref{subsec:man-machine}, a useful approach for communicating an image to a remote-server for some downstream CV task it to perform feature-matching on some intermediate layer of a DNN model. In fact, \cite{hyomin} provided theoretical and experimental proof that feature matching is preferable, in terms of rate distortion, to transmitting an entire image to the cloud. We build on this proof and show that the rate-distortion function when matching deeper layers is lower (better) compared to matching earlier ones.  To provide this proof we first need to define some notation. Let our CV model, with a task output previously denoted $T$ be defined as $f(\cdot)$, so that for a given input $\mathbf{X}$ we have $T=f(\mathbf{X})$. Next, let the mapping from the input $\mathbf{X}$ to a set of intermediate features $\mathbfcal{Y}_1\equiv\mathcal{F}^{(l_1)}$ be $g_1(\cdot)$ so that $\mathbfcal{Y}_1 = g_1(\mathbf{X})$, and the mapping from $\mathbfcal{Y}_1$ to the output $T$ be $h_1(\cdot)$ so that $T=h_1(\mathbfcal{Y}_1)$.  Next, define a second, deeper set of intermediate features $\mathbfcal{Y}_2\equiv\mathcal{F}^{(l_2)}$ (with $l_2>l_1$) and mappings $g_2(\cdot)$ and $h_2(\cdot)$ such that $\mathbfcal{Y}_2 = g_2(\mathbfcal{Y}_1)$ and $T = h_2(\mathbfcal{Y}_2)$. Note that under this notation, $f = h_1\circ g_1$ and $h_1 = h_2\circ g_2$. This notation is illustrated graphically below.
\begin{singlespace}
\begin{equation*}
\begin{tikzpicture}[
  inner sep=0pt,
  outer sep=0pt,
  baseline=(x.base),
]
  \tikzstyle{arrow} = [thin,->,>=stealth]
  
  \path[every node/.append style={anchor=base west}]
    (0, 0)
    \foreach \name/\code in {
      x/ \mathbf{X}~,
      tmp/\,\qquad,
      y1/ ~\mathbfcal{Y}_1~,
      tmp/\,\qquad,
      y2/ ~\mathbfcal{Y}_2~,
      tmp/\,\qquad,
      t/ ~T%,
    } {
      node (\name) {$\code$}
      (\name.base east)
    }
  ;
  \path[
    every node/.append style={
      anchor=base,
      font=\scriptsize,
    },
  ]
    % Annotation: f
    (x.base) -- node[above=3.0\baselineskip] (f) {\small $f$} (t)
    % 2 annotations: g, h
    (x.base) -- node[below=0.2\baselineskip] (g1) {\small $g_1$} (y1)
    (y1.base) -- node[above=1.8\baselineskip] (h1) {\small $h_1$} (t)
    (y1.base) -- node[below=0.2\baselineskip] (g2) {\small $g_2$} (y2)
    (y2.base) -- node[below=0.2\baselineskip] (h2) {\small $h_2$} (t)
  ;
  \draw [arrow] (x) -- (y1);
  \draw [arrow] (y1) -- (y2);
  \draw [arrow] (y2) -- (t);
    
  \begin{scope}[
    >={Stealth[length=5pt]},
    thin,
    rounded corners=2pt,
    shorten <=.3em,
    shorten >=.3em,
  ]
    
    \def\GebArrow#1#2#3{
      \draw[->]
        (#2.south) ++(0, -.3em) coordinate (tmp)
        (#1) |- (tmp) -| (#3)
      ;%
    }
    \GebArrow{x}{f}{t}
    \GebArrow{y1}{h1}{t}
  \end{scope}
\end{tikzpicture}
\label{eq:mappings}
\end{equation*}
\end{singlespace}

In the case of compression for machines, we are mainly interested in the output of our CV task, and thus we measure the distortion at $T$. By denoting this as $d_T$,  we have:
\begin{equation}
\begin{aligned}
\tilde{d}(\mathbf{x},\hat{\mathbf{x}}) = & d_T\left(f(\mathbf{x});f(\hat{\mathbf{x}})\right), \\ 
\tilde{d}(\mathbf{y}_1,\hat{\mathbf{y}}_1) = & d_T\left(h_1(\mathbf{y}_1),h_1(\hat{\mathbf{y}}_1)\right), \\ 
\tilde{d}(\mathbf{y}_2,\hat{\mathbf{y}}_2) = & d_T\left(h_2(\mathbf{y}_2),h_2(\hat{\mathbf{y}}_2)]\right).
\end{aligned}
\label{eq:distortion}
\end{equation}
It is important to note that the choice of distortion metric $d_T(\cdot,\cdot)$ depends on the CV task, and might generally differ from the metrics used for human vision. Next, we define the set of all approximations and their respective conditional distributions, which achieve a task distortion of at most $D$ as:
\begin{equation}
    \mathcal{P}_{\mathbf{X}}(D)=\left\{p\left(\hat{\mathbf{x}}|\mathbf{x}\right)~:~\mathbb{E}\left[\tilde{d}\big(\mathbf{X},\widehat{\mathbf{X}}\big)\right]\leq D\right\}.
    \label{eq:quant_X}
\end{equation}
Using~(\ref{eq:distortion}), we can obtain equivalent formulations for $\mathcal{P}_{\mathbfcal{Y}_1}(D)$ and $\mathcal{P}_{\mathbfcal{Y}_2}(D)$. This notation allows us to rewrite~(\ref{eq:RD_function}) as:
\begin{equation}
    R_{\mathbf{X}}(D) = \min_{p(\hat{\mathbf{x}}|\mathbf{x})~\in~\mathcal{P}_{\mathbf{X}}(D)} I(\mathbf{X};\widehat{\mathbf{X}}).
\end{equation}
Similarly, we can write equivalent formulations for $R_{\mathbfcal{Y}_1}(D)$ and $R_{\mathbfcal{Y}_2}(D)$.
% \begin{equation}
%     \begin{aligned}
%     R_{\mathbf{X}}(D) = & \min_{p(\hat{\mathbf{x}}|\mathbf{x})~\in~\mathcal{P}_{\mathbf{X}}(D)}& I(\mathbf{X};\widehat{\mathbf{X}}), \\
%     R_{\mathbfcal{Y}_1}(D) = & \min_{p_1(\hat{\mathbf{y}}_1|\mathbf{y}_1)~\in~\mathcal{P}_{\mathbfcal{Y}_1}(D)}& I(\mathbfcal{Y}_1;\widehat{\mathbfcal{Y}}_1), \\
%     R_{\mathbfcal{Y}_2}(D) = & \min_{p_2(\hat{\mathbf{y}_2}|\mathbf{y}_2)~\in~\mathcal{P}_{\mathbfcal{Y}_2}(D)}& I(\mathbfcal{Y}_2;\widehat{\mathbfcal{Y}}_2).
%     \end{aligned}
% \end{equation}
We are now ready to state our first result.
\begin{theorem}
\label{thm:lossy}
For any distortion level $D\geq 0$, the minimum achievable rate for compressing $\mathbfcal{Y}_1$ is an upper bound to the minimum achievable rate for compressing $\mathbfcal{Y}_2$, that is, 
\begin{equation}
    R_{\mathbfcal{Y}_2}(D) \leq R_{\mathbfcal{Y}_1}(D).
    \label{eq:RD_bound}
\end{equation}
Equality occurs when $\mathbfcal{Y}_1$ can be recreated exactly from $\mathbfcal{Y}_2$.
\end{theorem}

\begin{proof}
In~\cite{hyomin}, the authors prove that for any one intermediate layer $\mathbfcal{Y}_1 = g_1(\mathbf{X})$, we have $R_{\mathbfcal{Y}_1}(D) \leq R_\mathbf{X}(D)$.
To prove our theorem, we simply need to show that we can replace $\mathbf{X}$ with $\mathbfcal{Y}_1$, as well as replace $\mathbfcal{Y}_1$ with $\mathbfcal{Y}_2$, all the while maintaining the conditions of the proof. 

First, we recall that the only condition on the source $\mathbf{X}$ is that it must have a fixed distribution $p(\mathbf{x})$ in the sense that it does not depend on the approximation (or quantization) process. However, since $\mathbfcal{Y}_1$ is directly computed from $\mathbf{X}$, its own distribution, which we denote $p_1(\mathbf{y}_1)$, induced by $\mathbfcal{Y}_1 = g_1(\mathbf{X})$ is also fixed (because $g_1$ is fixed). 

To finish, we want to show that the structure relating $\mathbf{X}$, $\mathbfcal{Y}_1$, and $T$, is maintained between $\mathbfcal{Y}_1$, $\mathbfcal{Y}_2$, and $T$. However, because there are no special requirements on the functions $f,g_1,h_1$ in the original formulation (not even that they are deterministic), this is trivial. To do this, we can simply replace the function $f$ with $h_1$, $g_1$ with $g_2$, and $h_1$ with $h_2$, thus concluding our proof.
\end{proof}
% \vspace{-0.1cm}

% \color{violet}

In~\cite{hyomin}, the authors define a two-layer\footnote{The word layer here refers to the base and enhancement portion of the model, and not to DNN layers.} network where the \emph{base} representation is used to compress an input for object detection by a YOLOv3 vision backend \cite{Redmon2018_yolov3}, and the \emph{enhancement} is used (together with \emph{base}) for input reconstruction. To compress the image for object detection, the authors perform feature matching on the 13\textsuperscript{th} layer of YOLOv3, %The choice of layer is not explicitly explained, however, considering Theorem~\ref{thm:lossy} and the architecture of YOLOv3~\cite{Redmon2018_yolov3}, this is likely due to the $\mathcal{F}^{(13)}$ being the deepest single-stream feature in YOLOv3. %for which YOLOv3 remains a single-stream model.
because $\mathcal{F}^{(13)}$ is the deepest single-stream feature in YOLOv3.

In their three-layer model, the authors in~\cite{hyomin} utilize the \emph{base} and first \emph{enhancement} layers to perform feature matching on more complex architectures of Faster R-CNN~\cite{ren2015faster} and Mask R-CNN~\cite{he2017mask}, which are used for object detection and segmentation, respectively. There, because of the multi-stream structure of a feature-pyramid network (which is part of both R-CNN architectures), they use an alternative approach to feature matching. In this method, the target feature is reconstructed in two steps. First, using an LST, a single-stream feature $\mathcal{F}^{(l_0)}$ is reconstructed for some small $l_0$. Next, using pre-trained portions of the CV model (denoted $H_k$, and collectively named the CV mid-model), several multi-stream features  $\mathcal{F}^{(l_k)} = H_k(\mathcal{F}^{(l_0)}), k=1,2,..., K$ are recreated. For clarity, when using this approach we refer to $l_0$ as the partition point~\cite{kang2017neurosurgeon}, and $\mathcal{F}^{(l_k)}, k=1,2,..., K$ as downstream features.

% Using the two-step approach can be thought of as recreating the  downstream features $\mathcal{F}^{(l_k)}, k=1,2,..K$ from earlier partition point feature $\mathcal{F}^{(l_0)}$. In another sense, we can think of it as reconstructing $\mathcal{F}^{(l_0)}$ using distortion measured at $\mathcal{F}^{(l_k)}, k=1,2,..K$. 

Evaluating distortion in the manner used in Theorem~\ref{thm:lossy}, is equivalent to using the two-step approach to reconstructing the task output $T$, from the intermediate features $\mathbfcal{Y}_1$ or $\mathbfcal{Y}_2$. Using this distortion during training is impractical because it would require a labeled dataset of uncompressed images for each of the desired tasks. Unfortunately, datasets containing uncompressed images with CV task labels are rare, %and existing ones~\cite{takehiro} contain relatively few images, 
making supervised training of a scalable model impractical. Instead, both models in~\cite{hyomin} use mean squared error (MSE) of the reconstructed features to evaluate distortion during training. We will prove next, that the use of deeper layers in this practical setting is still beneficial.

To evaluate the effect of measuring distortion on intermediate features, we need to extend the previously defined notation. First, we define the distortion measured at some variable $\mathbfcal{Z}$ as $d_\mathbfcal{Z}(\mathbf{z},\hat{\mathbf{z}})$. Next, the set of all (including two-step) reconstructions of \emph{$\mathbfcal{Y}_2$ from $\mathbfcal{Y}_1$} that achieve a distortion of no more than $D$ (as measured at $\mathbfcal{Y}_2$) is defined as:
\begin{equation}
\resizebox{.9\hsize}{!}{$\mathcal{P}_{\mathbfcal{Y}_{21}}(D)=\left\{p_{21}\left(\hat{\mathbf{y}}_2|\mathbf{y}_1\right):\mathbb{E}\left[d_{\mathbfcal{Y}_2}\left(g_2(\mathbfcal{Y}_1),\widehat{\mathbfcal{Y}}_2)\right)\right]\leq D\right\}$}
    \label{eq:cross_quant}
\end{equation}
Analogously, the rate-distortion function of $\mathbfcal{Y}_2$ from $\mathbfcal{Y}_1$ is:
\begin{equation}
    R_{\mathbfcal{Y}_{21}}(D) =  \min_{p_{21}(\widehat{\mathbf{y}}_2|\mathbf{y}_1)~\in~\mathcal{P}_{\mathbfcal{Y}_{21}}(D)} I(\mathbfcal{Y}_1;\widehat{\mathbfcal{Y}}_2)
\label{eq:cross_RD}
\end{equation}
Note that $\mathcal{P}_{\mathbfcal{Y}_{11}}(D)$ is simply the conventional rate-distortion function of $\mathbfcal{Y}_1$. Finally, we say that a function $\phi$ has a distortion magnitude of $\delta$ if, for any conditional distribution $q(\hat{\mathbf{z}}|\mathbf{z})$, we have % defining $\mathbfcal{U} = w(\mathbfcal{Z})$ gives
$\mathbb{E}\left[d_{\phi(\mathbfcal{Z})}(\phi(\mathbfcal{Z}),\phi(\widehat{\mathbfcal{Z}}))\right] = \delta\cdot\mathbb{E}\left[d_\mathbfcal{Z}(\mathbf{z},\hat{\mathbf{z}})\right]$.\footnote{Note that %it hard to directly enforce a constraint on the distortion magnitude for a model, but 
commonly used operations, such as batch normalization, are designed to maintain relatively uniform feature magnitude in each layer, which in turn promotes distortion magnitude close to $\delta\!=\!1$. Alternatively, the distortion function $d_{\phi(\mathbfcal{Z})}$ can be scaled to achieve this behaviour in practice.} We are now ready to state and prove our next result.

%Alon - this is placed here to make it fit on the desired page, even though it appears somewhat strange.
\begin{figure*}[htbp]
\centerline{\includegraphics[width = 0.95\textwidth]{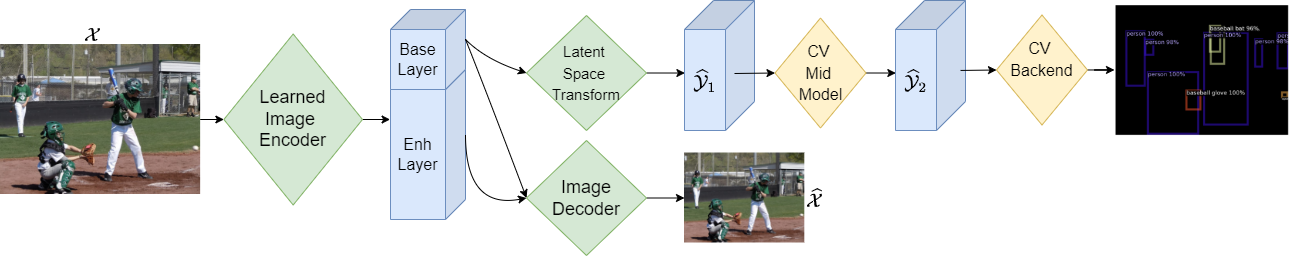}}
\caption{Block diagram of two-step feature matching.}
\label{fig:block}

\end{figure*}

\begin{theorem}
\label{thm:distortion}
Given a function $g_2$ with a distortion magnitude of $1$, the minimum achievable rate for compressing $\mathbfcal{Y}_2$ from $\mathbfcal{Y}_1$ is upper bounded by the minimum achievable rate for compressing $\mathbfcal{Y}_1$ from itself. That is: 
\begin{equation}
    R_{\mathbfcal{Y}_{21}}(D) \leq R_{\mathbfcal{Y}_{11}}(D).
    \label{eq:RD_bound_distortion}
\end{equation}
\end{theorem}

\begin{proof}
We begin by assuming we have an approximation $\widehat{\mathbfcal{Y}}_1$, represented by $p_1^{*}(\hat{\mathbf{y}}_1|\mathbf{y}_1)$, which achieves the conventional rate-distortion function of  $\mathbfcal{Y}_1$. This means that $\mathbb{E}\left[d_{\mathbfcal{Y}_1}\left(\mathbfcal{Y}_1,\widehat{\mathbfcal{Y}}_1\right)\right]\leq D$ and also $I(\mathbfcal{Y}_1;\widehat{\mathbfcal{Y}}_1)= R_{\mathbfcal{Y}_{11}}(D)$. We then define $\tilde{\mathbfcal{Y}}_2 = g_2(\widehat{\mathbfcal{Y}}_1)$, which induces a conditional distribution $p_{21}(\tilde{\mathbf{y}}_2|\mathbf{y}_1)$. Because the distortion magnitude of $g_2$ is $\delta = 1$ , we know that $p_{21}(\tilde{\mathbf{y}}_2|\mathbf{y}_1)\in \mathcal{P}_{\mathbfcal{Y}_{21}}(D)$ as:
\begin{equation}
    \begin{aligned}
        \mathbb{E}\left[d_{\mathbfcal{Y}_2}\left(g_2(\mathbfcal{Y}_1),\tilde{\mathbfcal{Y}}_2\right)\right] = & \: \mathbb{E}\left[d_{g_2(\mathbfcal{Y}_1})\left(g_2(\mathbfcal{Y}_1),g_2(\widehat{\mathbfcal{Y}}_1)\right)\right]\\
        = & \: \mathbb{E}\left[d_{\mathbfcal{Y}_1}\left(\mathbfcal{Y}_1,\widehat{\mathbfcal{Y}}_1\right)\right] \leq D.
    \end{aligned}
\end{equation}
Next, we note that the conditional distribution $p_{21}(\tilde{\mathbf{y}}_2|\mathbf{y}_1)$ is induced by the the following Markov chain:
\begin{equation*}
\begin{tikzpicture}[
  inner sep=0pt,
  outer sep=0pt,
  baseline=(x.base),
]
  \tikzstyle{arrow} = [thin,->,>=stealth]
  
  \path[every node/.append style={anchor=base west}]
    (0, 0)
    \foreach \name/\code in {
      x/ \textcolor{white}{\widehat{\mathbf{X}}}\mathbfcal{Y}_1~,
      tmp/\,\qquad \qquad,
      x_hat/ ~\widehat{\mathbfcal{Y}}_1~,
      tmp/\,\qquad,
      t_tilde/ ~\widetilde{\mathbfcal{Y}}_2%
    } {
      node (\name) {$\code$}
      (\name.base east)
    }
  ;
  \path[
    every node/.append style={
      anchor=base,
      font=\scriptsize,
    },
  ]
    % 2 annotations: g, h
    (x.base) -- node[below=0.2\baselineskip] (p) {~~~$p_1^{*}(\hat{\mathbf{y}}_1|\mathbf{y}_1)$} (x_hat)
    (x_hat.base) -- node[below=0.2\baselineskip] (g) {$g_2$} (t_tilde)
  ;
  \draw [arrow] (x) -- (x_hat);
  \draw [arrow] (x_hat) -- (t_tilde);
\end{tikzpicture}
\label{eq:producing_T_tilde}
\end{equation*}
We apply the data processing inequality~\cite{Cover_Thomas_2006} to this chain to get $I(\mathbfcal{Y}_1;\tilde{\mathbfcal{Y}}_2) = I\left(\mathbfcal{Y}_1;g_2(\widehat{\mathbfcal{Y}}_1)\right) \leq I(\mathbfcal{Y}_1;\widehat{\mathbfcal{Y}}_1) = R_{\mathbfcal{Y}_{11}}(D)$.

Because we have already shown $p_{21}(\tilde{\mathbf{y}}_2|\mathbf{y}_1)\in \mathcal{P}_{\mathbfcal{Y}_{21}}(D)$ we can use the definition of the rate-distortion function of $\mathbfcal{Y}_2$ from $\mathbfcal{Y}_1$. As defined in~(\ref{eq:cross_RD}), $R_{\mathbfcal{Y}_{21}}(D)$ is the minimum of the mutual information across all conditional distributions in $\mathcal{P}_{\mathbfcal{Y}_{21}}(D)$ and thus we see that $R_{\mathbfcal{Y}_{21}}(D) \leq I(\mathbfcal{Y}_1;\tilde{\mathbfcal{Y}}_2)\leq R_{\mathbfcal{Y}_{11}}(D)$, thereby concluding our proof.
\end{proof}
% \vspace{-0.2cm}

% \color{black}
\section{Experiments}

\subsection{Experimental setup}
To verify that our results hold in a practical setting, we modify the two-layer model presented in~\cite{hyomin}. Utilizing the original training algorithm, we use the following toss formulation to ensure good performance in terms of both task performance as well as bit-rate:
\begin{equation}
    \mathcal{L} = R + \lambda(D_{enh} + w \cdot D_{base}).
    \label{eq:loss}
\end{equation}
Here $\mathcal{L}$ is the training loss, $R$ is the bit-rate as estimated by a learned entropy model inspired by~\cite{cheng2020learned}, $D_{enh}$ and $D_{base}$ are the distortions of the image reconstruction and feature matching, respectively (both use mean square error, MSE, as their metric), and $\lambda$ and $w$ are Lagrange multipliers. By training the model using various values of $\lambda$, we produce rate-distortion curves for object detection and input reconstruction. 

Similarly to the approach taken in the three-layer case, we train the two-layer model using two-step feature matching, which can be seen in Fig.~\ref{fig:block}. In practice, using this two-step approach is equivalent to changing the distortion $D_{base}$ from being measured at the partition point to measuring it at some downstream layer features. Based on the multi-resolution structure of YOLOv3, we explore several combinations of layers, seen in  Table~\ref{tbl:mid_models}, to use for feature matching during training. When using more than one layer we combine the distortions by first flattening the tensors in question, then concatenating them, before calculating the MSE. Note that layer 108 is the final layer of YOLOv3, and it's output consists of multiple tensors, corresponding to different resolutions.

\begin{table}[h]
\centering
\caption{Experimental Setup and Results}
\label{tbl:mid_models}
\smallskip\noindent
\resizebox{1.0\linewidth}{!}{%
\begin{tabular}{@{}cccc@{}}
\toprule
Method & Feature & Base & Enhancement  \\ 
 & Layers & BD-Rate[\%]& BD-PSNR [dB]  \\ \midrule
HEVC & N/A & $89.27$ & $-1.04$ \\
VVC & N/A & $66.40$ & $\mathbf{0}$ \\ 
Cheng \textit{et al.}\cite{cheng2020learned} & N/A & $54.46$ & $-0.22$  \\ \midrule
Choi \textit{et al.}\cite{hyomin}, $w =0.06$ & $13$ & $0$ & $-0.92$  \\
Choi \textit{et al.}\cite{hyomin}, $w =0.12$ & $13$ & $-13.89$ & $-1.37$  \\ \midrule
Proposed - Mid-Model 1 & $\{13, 38\}$ & $-12.11$ & $-0.88$  \\
Proposed - Mid-Model 2 & $\{13, 38, 63\}$ & $-13.87$ & $-0.83$  \\
Proposed - Mid-Model 3 & $108$ & $\mathbf{-50.55}$ & $-2.56$  \\
\end{tabular}}
\end{table}

We train our model following a similar scheme to the one presented in~\cite{hyomin}, using a batch size of 16, and training in two stages. At first, we train for 400 epochs using a dataset comprised of randomly cropped patches from both the CLIC~\cite{clic_dataset} and JPEG-AI~\cite{jpeg_ai_dataset} datasets. Then we replace the dataset to a subset of VIMEO-90K~\cite{xue2019video_vimeo} and proceed to train for another 350 epochs (the use of a subset is due to the large amount of models to be trained).  The model is trained using an Adam optimizer with a fixed learning rate of $10^{-4}$ for the first stage which is then reduced with a polynomial decay every 10 epochs during the second stage.

After training, we evaluate our models in terms of rate-distortion performance for both object detection (base) and reconstruction (enhancement) tasks. For the base task we use the validation set of the COCO-2014 dataset~\cite{COCO}, which we resize to $512\times 512$ following the procedure in~\cite{hyomin}. In accordance with common practice for object detection, we use the mean average precision (mAP) as our metric. For the enhancement task we use the Kodak  dataset~\cite{kodak_dataset}, and PSNR as our evaluation metric. As we have proven in our theoretical discussion, we expect our modifications to lead to improved rate-distortion performance in the computer vision task, perhaps at some cost to the reconstruction task.

We compare our model against the original two-layer model, which we retrain using the same data for a fair comparison. Since our method favors the base task slightly, we also train a second version of the two-layer model, using the original layers for distortion but with a larger Lagrange multiplier for $D_{base}$, $w=0.012$. Lastly, we also compare our model against three baselines:  HEVC~\cite{hevc_std_2019}, VVC~\cite{vvc_std}, and a learned compression model~\cite{cheng2020learned}, where an image is fully decoded on the server-side before being passed to YOLOv3 for detection. 

\subsection{Results}

We begin by reviewing the results for the base task, as seen in Fig.~\ref{fig:results_base}. We can see that, as expected, the use of deeper layers, has led to improved rate-distortion performance, with the largest improvement obtained using the final layer of YOLOv3. To quantify the differences in rate-distortion performance, we compute the BD-Rate metric~\cite{Bjontegaard}, which measures the average difference in bits required to achieve equal performance (in terms of mAP in this case) to a baseline (we use~\cite{hyomin}). Once again, using the final layer achieves the best performance. Interestingly, increasing the Lagrange multiplier for the base task ($w$ in~\ref{eq:loss}) to 0.012, had a comparable effect to that of our method using mid-model 1 or 2.

\begin{figure}[htbp]
        \centering
         \includegraphics[width=\linewidth]{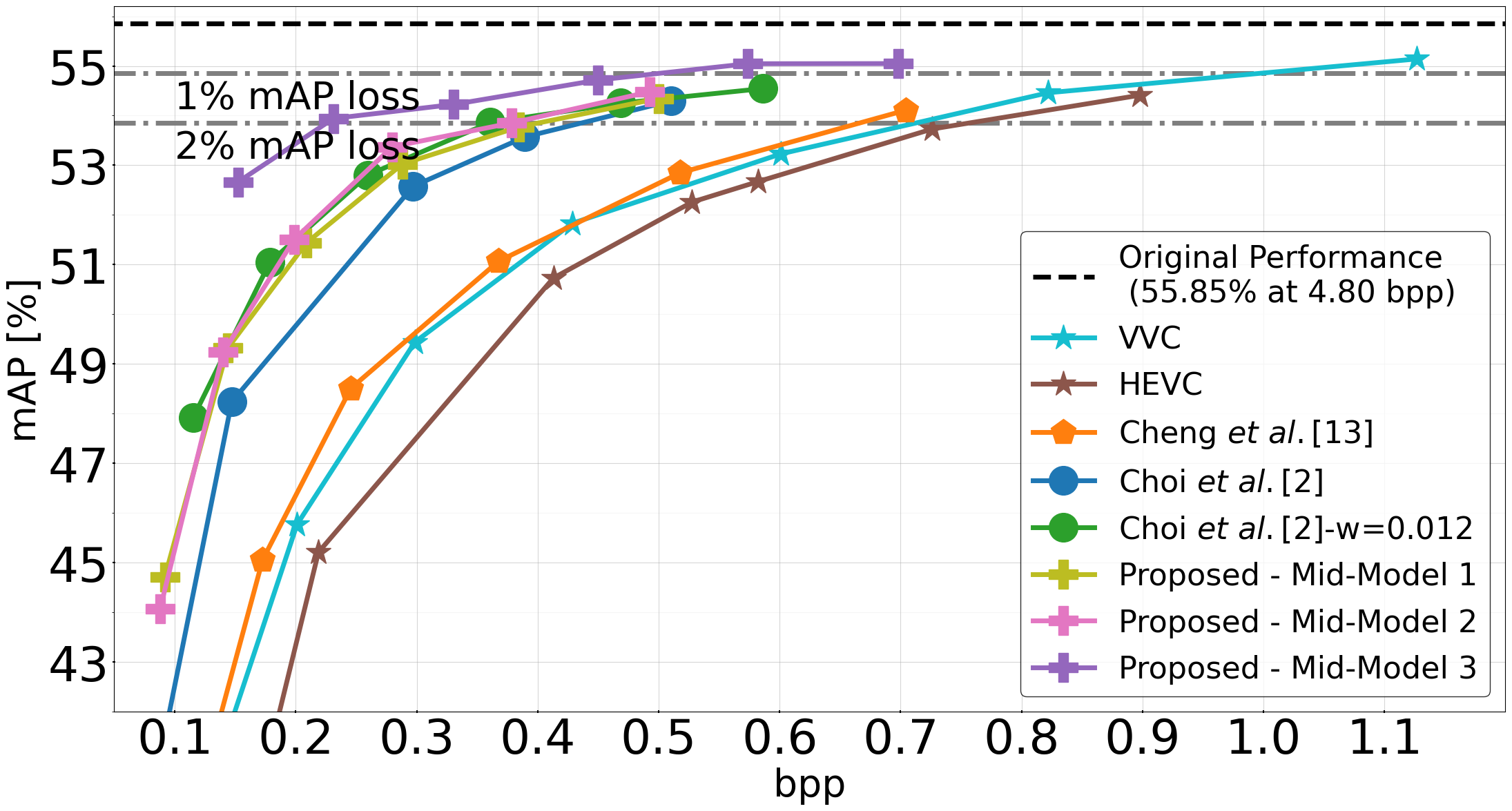}
         \begin{footnotesize}
         \caption{Mean Average Precision (mAP) vs. Bits per Pixel (bpp) for object detection using YOLOv3.}
         \label{fig:results_base}
         \end{footnotesize}
         
\end{figure}

Next, we observe our results on the enhancement task, shown in Fig.~\ref{fig:results_enh}. We summarise the performance using the BD-PSNR metric which measures the average difference in PSNR compared to a baseline (we use VVC~\cite{vvc_std}) at an equal bit-rate, seen in Table~\ref{tbl:mid_models}. Comparing the performance here to the base task we clearly see the trade-off between the two. This is most notable in the performance of our method using mid-model 3, which was the best method for object detection, and is the weakest here on image reconstruction. Importantly, our method using mid-model 1 or 2 performs slightly better than the modified version of \cite{hyomin} using $w=0.012$, which had comparable performance on the base task.

\begin{figure}[htbp] 
         \centering
         \includegraphics[width=\linewidth]{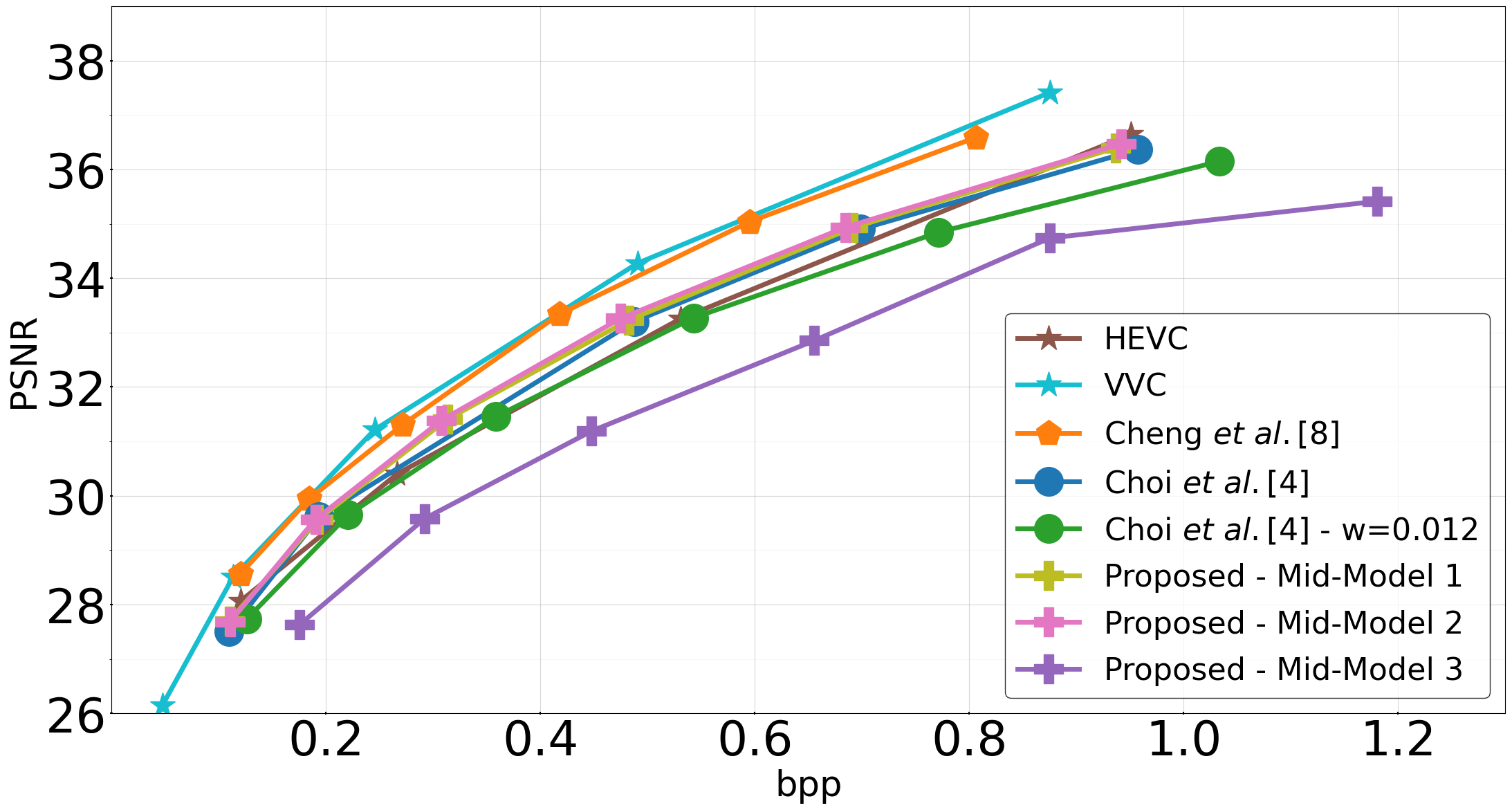}
         \begin{footnotesize}
         \caption{PSNR vs. Bits per Pixel (bpp) for image reconstruction.}
         \label{fig:results_enh}
         \end{footnotesize}
         
\end{figure}

% \vspace{-0.2cm}
\section{Summary and Discussion}

The growing prevalence of DNNs in the field of computer vision have lead to an increasing volume of images seen predominantly by machines. This, in turn, requires efficient coding methods for machines and a good understanding of their rate-distortion behaviour. In this paper we provided important theoretical background for understanding the implication of layer selection when encoding an image for feature matching in a downstream DNN model. We have proved that using deeper layers achieves superior rate-distortion performance compared to earlier ones. This is true both when evaluating distortion at the task level, as well as when the distortion is measured at some intermediate point.

To demonstrate the usefulness of our theoretical results, we used them to modify the training loss of an image coding method for humans and machines~\cite{hyomin}. The results clearly show that using our approach yields improvement in the rate-distortion performance of the machine side of the model, as expected. Importantly, the modifications do not require any change to the encoder side architecture, or to the decoding process, making our approach easy to implement in practice.  This is critical as it does not require more edge-device resources than simply using the learned image coding backbone~\cite{cheng2020learned}, on which it is based.

As might have been expected, the improved performance on the machine performance comes at some cost on the human side. Using a hyper-parameter ($w$), the framework in~\cite{hyomin} allows a designer to balance between performance on each side during training. When considering either model for a specific application, one might estimate the frequency at which images will need to be fully reconstructed and adjust the training procedure accordingly. Notably, when both models are adjusted to give equal rate-distortion performance on object detection, our model resulted in superior performance on image reconstruction compared with~\cite{hyomin}.

\bibliographystyle{IEEEtran}
\small
\bibliography{refs}
% \vspace{12pt}
% \color{red}
% IEEE conference templates contain guidance text for composing and formatting conference papers. Please ensure that all template text is removed from your conference paper prior to submission to the conference. Failure to remove the template text from your paper may result in your paper not being published.

\end{document}